\documentclass[conference]{IEEEtran}
\IEEEoverridecommandlockouts

\usepackage{cite}
\usepackage{comment}
\usepackage{amsmath,amssymb,amsfonts}
\usepackage{algorithmic}
\usepackage{graphicx}
\usepackage{textcomp}
\usepackage{xcolor}
\def\BibTeX{{\rm B\kern-.05em{\sc i\kern-.025em b}\kern-.08em
    T\kern-.1667em\lower.7ex\hbox{E}\kern-.125emX}}

\usepackage{subcaption}
\usepackage{caption} 

\usepackage[hidelinks]{hyperref}
\usepackage{amsthm}
\newtheorem{theorem}{Theorem}
\newtheorem{lemma}{Lemma}
\allowdisplaybreaks 
\begin{document}

\title{Update Estimation and Scheduling for Over-the-Air Federated Learning with Energy Harvesting Devices\\
\thanks{Furkan Bagci's work is supported by Türk Telekomünikasyon A.S. within the framework of the 5G and Beyond Joint Graduate Support Programme coordinated by the Information and Communication Technologies Authority. \\
\indent The authors acknowledge support from TUBITAK through CHIST-ERA
project SONATA (CHIST-ERA-20-SICT-004, funded by TUBITAK, Turkey
Grant 221N366). \\
\indent The work of Mohammad Kazemi was supported by UKRI under the U.K. government’s Horizon Europe Funding Guarantee under Grant 101103430.
}
}

\author{\IEEEauthorblockN{Furkan Bagci\textsuperscript{1}, Busra Tegin\textsuperscript{1}, Mohammad Kazemi\textsuperscript{2} and Tolga M. Duman\textsuperscript{1}}
\IEEEauthorblockA{\textsuperscript{1}\textit{Dept. of Electrical and Electronics Engineering, Bilkent University}, Ankara, Turkey \\
\textsuperscript{2}\textit{Dept. of Electrical and Electronic Engineering, Imperial College London}, London, UK \\
\{bagci, btegin, duman\}@ee.bilkent.edu.tr, mohammad.kazemi@imperial.ac.uk}

}

\maketitle

\begin{abstract}

We study over-the-air (OTA) federated learning (FL) for energy harvesting devices with heterogeneous data distribution over wireless fading multiple access channel (MAC). To address the impact of low energy arrivals and data heterogeneity on global learning, we propose user scheduling strategies. Specifically, we develop two approaches: 1) entropy-based scheduling for known data distributions and 2) least-squares-based user representation estimation for scheduling with unknown data distributions at the parameter server. Both methods aim to select diverse users, mitigating bias and enhancing convergence. Numerical and analytical results demonstrate improved learning performance by reducing redundancy and conserving energy.

\end{abstract}


\section{Introduction}
With the growing number of devices collecting high-quality data, vast amounts of practical information are available for machine learning (ML) applications. Traditional ML approaches require centralized data sharing, demanding significant resources and raising privacy concerns. To address these issues, \textit{Federated Learning} (FL) enables collaborative model training without sharing local data, ensuring privacy, low latency, and improved learning quality \cite{mcmahan2017commef,nguyen2021FLIot}. However, FL relies on iterative transmission of local updates from mobile users (MUs) to the parameter server (PS), making communication bandwidth a key bottleneck \cite{yang2020ota}. Over-the-air (OTA) computation addresses this issue by leveraging wireless multiple access channel (MAC) to aggregate updates directly during transmission \cite{yang2020ota, amiri2020dsgd}. Multiple receive antennas with combining techniques in OTA FL can mitigate channel effects and alleviate fading, enabling convergence even with blind transmitters lacking channel state information (CSI) \cite{amiri2020BFL}.

The deployment of FL faces many challenges, such as hardware impairments, data distribution issues, and limited energy for MUs. To address energy constraints, energy harvesting (EH) devices are widely used to obtain energy from the environment, with studies exploring the channel capacity with various battery sizes \cite{sennur2015EH}. 
For FL with OTA EH devices, studies have examined joint user selection and receive beamforming \cite{chen2023joint}, weighted averaging to prevent bias \cite{ozan2022ehfl}, transceiver optimization \cite{an2024online}, and Markov decision processes for joint scheduling and power management \cite{zhang2024mdp}.

Studies on FL with non-independent and identically distributed (non-i.i.d.) data show that heterogeneous local datasets significantly impact model accuracy and convergence \cite{zhao2018noniid, li2019convergence}. In \cite{wang2020rl}, a deep reinforcement learning-based user selection framework is utilized to mitigate bias by exploiting the relationships between user updates and data distributions. In \cite{fraboni2021clusteredsampling}, clustered sampling is used to schedule users based on cosine similarity, improving representation. In \cite{balakrishnan2022diverse}, federated averaging with diverse user selection is employed to approximate full participation gradients, reducing redundancy from similar data. These studies assume user updates reflect data distributions, relying on separate transmission of user updates. In contrast, our OTA FL setup uses noisy aggregated updates transmitted simultaneously by all scheduled users, eliminating the need for separate-transmission resource allocation. 

In this work, we explore diverse user selection for FL with energy harvesting devices and heterogeneous data distribution by scheduling active users based on their data distribution using over-the-air transmission. Firstly, we select a subset of users that achieves a uniform data distribution using exact data distribution information to ensure the aggregate of selected users' updates approximates that of all users. Next, since knowing data distribution raises privacy concerns, we estimate user updates from aggregated signals on the PS via least-squares estimation (LSE), enabling user selection without revealing data distribution information. Unlike existing works, we use OTA aggregated signals instead of individual user updates for transmission efficiency along with participation information to estimate local model update representations. These representative updates are utilized to group users into clusters, improving learning performance and identifying redundant information in the system.

The paper is organized as follows. Section II introduces the FL setup and EH devices. Section III analyzes OTA FL for EH devices with a unit-sized battery and provides convergence analysis. In Section IV, we propose update estimation and user scheduling policies, with numerical results presented in Section V. We conclude the paper in Section VI.

\textit{Notations:} For vectors $\mathbf{x}$ and $\mathbf{y}$ of the same dimension, $\mathbf{x} \circ \mathbf{y}$ denotes element-wise product. We define $[i] = \{1, \dots, i\}$.

\section{System Model}

We consider an FL system with heterogeneous data distribution for $M$ energy harvesting devices. We assume that the mobile users do not have CSI and transmit their updates to the PS through a fading MAC using over-the-air transmission. The PS employs $K$ receive antennas to align the received signals in the absence of CSI at the transmitters (CSIT) using aggregated channel information.

In FL, the primary objective is to minimize a global loss function, denoted as $F(\boldsymbol{\theta })$, collaboratively across $M$ devices, where $\boldsymbol{\theta } \in \mathbb{R}^{2N}$ represents the parameters of the global model to be optimized. The global loss function is defined as
\begin{equation}\label{gl_loss}
F(\boldsymbol{\theta }) = \frac{1}{B} \sum_{m=1}^{M} \frac{\left|\mathcal{B}_{m}\right|}{B} F_{m}(\boldsymbol{\theta}),
\end{equation}
where $F_{m}(\boldsymbol{\theta})$ represents the average empirical local loss for the $m$-th user with model parameters $\boldsymbol{\theta}$. For $m$-th user with local dataset $\mathcal{B}_{m}$, for $m \in [M] $, and $ B \triangleq \sum_{m=1}^{M} \left| \mathcal{B}_{m} \right| $, we have
\begin{equation} \label{lc_loss}
    F_{m} \left ({\boldsymbol {\theta } }\right) = \frac {1}{\left| \mathcal{B}_{m} \right|} \sum \limits _{\boldsymbol {u} \in \mathcal {B}_{m}} f \left ({\boldsymbol {\theta }, \boldsymbol {u} }\right),
\end{equation}
with $f(\boldsymbol{\theta }, u)$ as the empirical loss function corresponding to the $u$-th data sample in the local dataset $B_m$.

In FL with EH devices, unlike traditional FL, the limited energy availability can cause some users to lack sufficient energy to perform local computations and transmissions. Consequently, contributions will only come from the MUs that have enough energy and are selected based on the scheduling policy. At each global iteration $t$, the PS broadcasts the latest global model, $\boldsymbol{\theta }(t)$. In response, the selected MUs perform $\tau$ iterations of local stochastic gradient descent (SGD) to minimize their individual local loss functions, $F_m(\boldsymbol{\theta })$, for $m \in \mathcal{S}(t)$, where $\mathcal{S}(t)$ is the set of scheduled users. Subsequently, the model updates obtained by the mobile users are transmitted back to the PS to contribute to the global learning process.

To compute the local model updates, the $m$-th user at the $i$-th local and $t$-th global iteration performs the local iterations of SGD with the following update rule: 
\begin{equation} \label{lc_update_rule}
    \boldsymbol{\theta}_{m}^{i+1}(t)=\boldsymbol{\theta}_{m}^{i}(t)-\eta_{m}^{i}(t)\nabla F_{m}(\boldsymbol{\theta}_{m}^{i}(t), \xi_{m}^{i}(t)),
\end{equation}
where $i \in [\tau]$, $\eta_{m}^{i}(t)$ is the learning rate and $\nabla F_{m}\left(\boldsymbol{\theta}_{m}^{i}(t), \xi_{m}^{i}(t)\right)$ represents the stochastic gradient estimate for the $\boldsymbol{\theta}_{m}^{i}(t)$ and the local mini-batch sample $\xi_{m}^{i}$ randomly chosen from the local dataset $\mathcal{B}_{m}$. 

After completing the local SGD steps, the $m$-th user computes the model update, which is aimed to be shared with the PS as 
\begin{equation} \label{lc_update}
    \Delta\boldsymbol{\theta}_{m}(t)=\boldsymbol{\theta}_{m}^{\tau}(t)-\boldsymbol{\theta}_{m}^{1}(t).
\end{equation}

Using over-the-air transmission over a fading MAC, the received signal at the $k$-th antenna of the PS at iteration $t$ is given as
\begin{equation} \label{rec_signal_0}
    \boldsymbol{y}_{PS,k}(t)=\sum_{m\in \mathcal{S}(t)}\boldsymbol{h}_{m,k}(t)\circ \boldsymbol{x}_{m}(t)+\boldsymbol{z}_{PS,k}(t),  
\end{equation}
where $\boldsymbol{x}_{m}(t)$ is the signal transmitted by the $m$-th user, and $\boldsymbol{h}_{m,k}(t)$ is the independent and identically distributed (i.i.d.) channel gains from the $m$-th user to the $k$-th antenna with entries ${h}_{m,k}^{n}(t) \sim \mathcal{C N}(0, \sigma_{h}^{2})$.  Similarly, ${z}_{PS,k}^{n}(t)$ denotes the $n$-th entry of the channel noise, $\boldsymbol{z}_{PS,k}(t)$, which is i.i.d. circularly symmetric additive white Gaussian noise (AWGN) and is distributed according to $\mathcal{C N}(0, \sigma_{z}^{2})$.

The PS uses the received signals from $K$ antennas to update the global model as
\begin{equation} \label{gl_update_part}
    \boldsymbol{\theta}_{PS}(t+1)=\boldsymbol{\theta}_{PS}(t)+ \Delta \hat{\boldsymbol{\theta}}_{PS}(t),
\end{equation}
where $\boldsymbol{\theta }_{\text{PS}}(t)$ represents the global model vector at the server at global iteration $t$ and $\Delta \hat{\boldsymbol{\theta}}_{PS}(t)$ is the noisy estimate of the average of the local updates. Note that if there were no noise or fading, the average of the local updates would be 
\begin{equation} \label{gl_update_error_free}
    \Delta \boldsymbol{\theta}_{PS}(t)=\frac{1}{\left| \mathcal{S}(t) \right| }\sum_{m\in \mathcal{S}(t)}\Delta\theta_{m}(t). 
\end{equation}

\subsection{EH Devices}

We study FL with OTA for EH devices, each with a unit-sized battery. At each global iteration, devices harvest energy with varying success, storing it for future use following the harvest-store-use approach \cite{EHsensors}. Surplus energy is lost if the battery is full during energy arrivals. Local SGD and update transmissions consume one energy unit global per iteration, highlighting stochastic energy availability as a key constraint.

We consider a Bernoulli energy arrival process where, at each global iteration $t$, the $m$-th user receives unit energy with probability $p_{e}^{m}(t)$. Active users consume available battery energy, while inactive users store it for future use. Energy arrivals are shared with the PS after each iteration, as in \cite{an2024online, chen2023joint}.

\section{OTA FL with EH Devices with Unit Battery}

In this section, we consider FL with OTA aggregation, where only active users contribute to the iterations due to limited energy arrivals and provide a brief convergence analysis.

The model updates of the scheduled users, $\Delta \boldsymbol{\theta}_{m}^{cx}(t) \in \mathbb{C}^{N}$, $m\in \mathcal{S}(t)$, are transmitted as complex signals at iteration $t$, represented as 
\begin{subequations}
\begin{align} 
    \Delta \boldsymbol{\theta}_{m}^{r e}(t) &\triangleq \left[\Delta \theta_{m}^{1}(t), \Delta \theta_{m}^{2}(t), \ldots, \Delta \theta_{m}^{N}(t)\right]^{T} \label{complex_update_re}, \\
    \Delta \boldsymbol{\theta}_{m}^{i m}(t) &\triangleq \left[\Delta \theta_{m}^{N+1}(t), \Delta \theta_{m}^{N+2}(t) , \ldots, \Delta \theta_{m}^{2N}(t)\right]^{T} \label{complex_update_im}, \\
    \Delta \boldsymbol{\theta}_{m}^{cx}(t) &\triangleq \Delta \boldsymbol{\theta}_{m}^{r e}(t) + j \Delta \boldsymbol{\theta}_{m}^{i m}(t)
    \label{complex_update}.
\end{align}
\end{subequations}

Using the channel output at each antenna as defined in (\ref{rec_signal_0}), the PS combines signals from $K$ antennas using the sum of the channel gains from the scheduled users to each antenna as
\begin{equation} \label{combined_signal}
    \boldsymbol{y}_{PS}(t)= \frac{1}{K}\sum_{k=1}^{K}\left(\sum_{m\in \mathcal{S}(t)} \boldsymbol{h}_{m,k}(t)\right)^{*}\circ \boldsymbol{y}_{PS,k}(t).
\end{equation}
where the received signals $\boldsymbol{y}_{PS,k}(t)$'s are given in (\ref{rec_signal_0}) for the transmitted signal $\Delta\boldsymbol{\theta}_{m}^{cx}(t)$.

The $n$-th symbol of (\ref{combined_signal}) can be partition into three signals
\begin{align} 
    y_{PS}^{n}(t)=&\underset{y_{PS}^{n,sig}(t)(\text{signal term})}{\underbrace{\sum_{m\in\mathcal{S}(t)}\left(\frac{1}{K}\sum_{k=1}^{K}\vert h_{m,k}^{n}(t)\vert ^{2}\right)\Delta\theta_{m}^{n,cx}(t)}} \notag\\ &+\underset{y_{PS}^{n,int}(t)(\text{interference term})}{\underbrace{\frac{1}{K}\sum_{m\in\mathcal{S}(t)}\underset{m^{\prime}\neq m}{\sum_{m^{\prime}\in\mathcal{S}(t)}}\sum_{k=1}^{K}(h_{m,k}^{n}(t))^{*}h_{m^{\prime},k}^{n}(t)\Delta\theta_{m^{\prime}}^{n,cx}(t)}} \notag \\ &+\underset{y_{PS}^{n,noise}(t)(\text{noise term})}{\underbrace{\frac{1}{K}\sum_{m\in\mathcal{S}(t)}\sum_{k=1}^{K}(h_{m,k}^{n}(t))^{*}z^{n}_{PS,k}(t)}} \label{rec_signal}.
\end{align}

As shown in \cite{amiri2020BFL}, the variance of the interference coefficient $\Delta \theta_{m^{\prime}}^{n, cx}(t)$ decreases with the number of antennas $K$. Hence, a sufficient number of antennas allows for accurate estimation and recovery of noisy aggregated updates as
\begin{subequations} \label{global_update}
\begin{align}
    \Delta\hat{\boldsymbol{\theta}}_{PS}^{n}(t)&=\frac{1}{ \left| \mathcal{S}(t) \right| \sigma_{h}^{2}}{\text{Re}}\{y_{PS}^{n}(t)\}, \\
    \Delta\hat{\boldsymbol{\theta}}_{PS}^{n+N}(t)&=\frac{1}{ \left| \mathcal{S}(t) \right| \sigma_{h}^{2}}{\text{Im}}\{y_{PS}^{n}(t)\},
\end{align}
\end{subequations}
to update the global model, as in (\ref{gl_update_part}).

\subsection{Convergence Analysis}
 
In this section, we provide convergence analysis for the OTA FL with EH devices and no CSIT by upper-bounding the gap between our model estimate and the optimal model.

The optimal solution minimizing (\ref{gl_loss}) is $\boldsymbol{\theta}^{*} \triangleq \arg \min_{\boldsymbol{\theta}} F(\boldsymbol{\theta})$, with optimal loss $F^{*} = F\left(\boldsymbol{\theta}^{*}\right)$. For user $m \in [M]$, the optimal local model is $\boldsymbol{\theta}_{m}^{*} \triangleq \arg \min_{\boldsymbol{\theta}_{m}} F_{m}(\boldsymbol{\theta}_{m})$, with corresponding loss $F_{m}^{*} = F\left(\boldsymbol{\theta}_{m}^{*}\right)$.

\subsubsection{Preliminaries}

The amount of bias and heterogeneity across devices is represented by the following non-negative parameter 
\begin{equation} \label{Bias}
    \Gamma = F^{*} - \sum_{m=1}^{M} \frac{B_{m}}{B} F^{*}_{m},
\end{equation}
A higher $\Gamma$ indicates significant non-i.i.d. data distribution, while $\Gamma \to 0$ reflects near i.i.d..

We consider the same learning rate across users and local iterations, \(\eta_m^i(t) = \eta(t)\), but allow it to vary between global iterations. The local model update at the $m$-th user for global iteration $t$ and local iteration $i \in [\tau]$ is given as 
\begin{equation} \label{updatess}
    \boldsymbol {\theta }_{m}^{i+1} (t)\,\,- \boldsymbol {\theta }_{m}^{1} (t)\,\,= - \eta (t)~\sum \limits _{l=1}^{i} \nabla F_{m} \left ({\boldsymbol {\theta }_{m}^{l} (t), \xi _{m}^{l} (t)~}\right).
\end{equation}

To perform the convergence analysis, similar to the existing studies \cite{amiri2020BFL, busra2023, ozan2024}, we assume that the loss functions \(F_1, \ldots, F_M\) are all \(L\)-smooth and \(\mu\)-strongly convex. Also, it is assumed that the expected squared \(\ell_2\)-norm of the stochastic gradients is bounded; that is, for all \(i \in [\tau]\), \(m \in [M]\), and \(t\), we have $\mathbb{E}_{\xi} \left[ \left\| \nabla F_m \left( \boldsymbol{\theta}_m^i(t), \xi_m^i(t) \right) \right\|_2^2 \right] \leq G^2$.

\subsubsection{Convergence Rate}

For the convergence analysis of OTA FL with EH devices and data heterogeneity, we begin by analyzing the convergence by comparing the optimal model with our system model, where only a subset of users participate in each iteration. Using these findings, we guide user scheduling decisions to minimize the discrepancy in this bound. Our main result is as follows.

\begin{theorem}  \label{thm1} For \( 0 < \eta(t) \leq \min \left\{ 1, \frac{1}{\mu \tau} \right\}, \forall t \). We have
\begin{align}\label{conv1}&\hspace {-.5pc}\mathbb {E} \left [{{ \left \|{{ \boldsymbol {\theta } (t)- {\boldsymbol {\theta }}^{*} }}\right \|_{2}^{2} }}\right ] \leq \left({{\prod \limits _{i=0}^{t-1} A(i) }}\right) \left \|{{ {\boldsymbol {\theta }} (0) - {\boldsymbol {\theta }}^{*} }}\right \|_{2}^{2} \notag\\& \qquad\qquad\qquad\qquad\qquad {+\, \sum \limits _{j=0}^{t-1} B(j) \prod \limits _{i=j+1}^{t-1} A(i),}\end{align}
with
\begin{align} \label{conv2} A(i)\triangleq&1 - \mu \eta (i)~\left ({\tau - \eta (i) (\tau - 1) }\right),\\ B(i)\triangleq& \frac { \eta ^{2}(i) \tau ^{2} G^{2}}{K} + \frac {\sigma _{z}^{2}N}{\alpha _{i}^{2} K {\left| \mathcal{S}(i) \right|} \sigma _{h}^{2}} \notag\\&+ \left ({1+ \mu (1- \eta (i)) }\right) \eta ^{2}(i) G^{2} \frac {\tau (\tau -1)(2\tau -1)}{6} \notag\\&+ \eta ^{2}(i) (\tau ^{2} + \tau -1)  G^{2} + 2 \eta (i) (\tau - 1) \Gamma \notag\\
& +\left( \eta^2(t) \tau(\tau-1)LG + \eta(t)\tau \epsilon  \right)^{2} \notag\\
&+ \left( \eta^2(t) \tau(\tau-1)LG + \eta(t)\tau \epsilon \right) c,\label{conv3} \end{align}
for some constant $c \geq 0$.

\begin{proof} 
See Appendix A.
\end{proof}
\end{theorem}
We note that A$(i)$ represents the decay rate of the distance from the initial starting point to the optimal solution. In $B(i)$, the first two terms represent the transmission error due to the wireless fading MAC with blind transmitters, and the third and fourth terms are related to federated averaging. Additionally, we emphasize that the last two terms represent the error caused by partial user participation similar to \cite{balakrishnan2022diverse}, with $\epsilon$ in (\ref{conv3}) being the gradient approximation error and defined as follows 
\begin{align} \label{epsilon}
\epsilon \triangleq \left\| 
\frac{1}{M} \sum_{m=1}^{M} \nabla F_m(\theta_m(t) )
-
\frac{1}{\left| S(t) \right|} \sum_{m \in \mathcal{S}(t)}  \nabla F_m(\theta_m(t))   
\right\|_2.
\end{align}
Via user scheduling, our aim is to select the subset of users with the average local updates closest to the full participation case, thereby minimizing the error $\epsilon$ in (\ref{epsilon}). Hence, in the next section, we propose user selection and scheduling approaches that minimize the distance between the trained FL model and the optimal model, even in scenarios involving EH devices and highly non-i.i.d. data distributions.

\section{Update Estimation and User Scheduling}
In this section, we demonstrate that the data distribution of users is critical in the scheduling procedure, especially for highly non-i.i.d data distributions, and can be used to minimize the error bound related to the partial participation characteristics of EH devices. First, we propose an entropy-based user scheduling policy with known data distribution. Then, we extend our discussion to the unknown data distribution case and show that the data distribution characteristics can be estimated via least-squares to be used in user scheduling.

\subsection{Entropy-based User Scheduling with Known Data Distributions}

Assuming all MUs reveal their data distribution to the PS beforehand, our method selects a subset of users that effectively represents all data labels in the network. Based on this, the PS characterizes the label distribution of each user $m \in [M]$ as $L_{m} = \left[ l_{m,0}, l_{m,1}, \dots, l_{m,{N_c}-1} \right]$, where $N_c$ is the total number of classes, and $l_{m,{n_c}}$ represents the portion of the $m$-th user's data corresponding to label $n_c$. At each iteration, the PS computes the label distribution for all available user subsets as a probability mass function and selects the one with the highest Shannon entropy, indicating the most balanced label distribution available. While this strategy is similar to that in \cite{lutz2024entropybased}, we extend our approach to a more practical setup that incorporates OTA transmission, wireless fading MAC, and blind transmitters, showing the effectiveness of entropy-based user selection for EH devices under practical constraints.

\subsection{User Clustering and Scheduling with Unknown Data Distribution}

We consider a more realistic case, where the PS lacks data distribution information, ensuring a level of user privacy.

For user scheduling, we rely on the relationship between the model updates from users and their underlying data distribution, similar to studies in \cite{wang2020rl, fraboni2021clusteredsampling, balakrishnan2022diverse}. Unlike these studies, our approach is constrained to using a noisy estimate of the sum of updates provided by all selected users at a given iteration. 
We demonstrate that a representation of the user updates can be estimated at the PS, allowing users to be grouped into clusters based on the similarities between their representations to minimize error from partial participation in (\ref{epsilon}). This approach selects suitable users while preventing redundant information transfer and conserving energy under the constraints of EH devices.

To achieve this, we use LSE to create a representation of the updates. 
Over $T$ estimation iterations, the PS stores normalized global updates from (\ref{global_update}) while all the active users participate without scheduling, termed the \textit{estimation phase}. Note that we normalize user updates to unit norm to mitigate scale discrepancies. At the end of this estimation window, PS estimates representative updates based on stored global updates and participation information. We emphasize that the goal is to estimate a representation of user updates rather than recovering the individual updates themselves.

We define a matrix $\hat{\boldsymbol{\Theta}}_{PS}$, whose rows represent global model updates $\Delta\hat{\boldsymbol{\theta}}_{PS}(t)$ from (\ref{global_update}). For the $j$-th iteration with $j\le T$,  the $j$-th row of this matrix can be expressed as $\hat{\boldsymbol{\Theta}}_{PS, j} = \boldsymbol{A}_{j} \boldsymbol{\Theta}_{j} + \boldsymbol{N}_{j}^{'}$, where $\boldsymbol{A}_{j}$ is a binary participation vector with  $\boldsymbol{A}_j \in \{0,1\}^{1 \times M}$, and $\boldsymbol{\Theta}_{j} \in \mathbb{R}^{M \times 2N}$, with each row representing the local model update for a specific user $m \in [M]$, denoted as $\Delta \boldsymbol{\theta}_{j,m}$.
Additionally, $\boldsymbol{N}_{j}^{'} \in \mathbb{R}^{1 \times 2N}$, whose $d$-th element is denoted by ${N}_{j,d}^{'}$ for $d \in [2N]$, represents the effective noise arising from MAC fading, AWGN, and PS combining errors. $\hat{\boldsymbol{\Theta}}_{PS, j}$ is expressed as:
\begin{align}
\hat{\boldsymbol{\Theta}}_{PS, j} = \boldsymbol{A}_{j} \begin{bmatrix}
\Delta \boldsymbol{\theta}_{j,1} \\
\vdots \\
\Delta \boldsymbol{\theta}_{j,M}
\end{bmatrix}_{} + 
\begin{bmatrix}
{N}_{j,1}^{'} 
\quad \dots \quad 
{N}_{j,2N}^{'}
\end{bmatrix}_{}.
\end{align}

We also define $\boldsymbol{\Theta}_{rep} \in \mathbb{R}^{M \times 2N}$ as a representation of local updates. Using this, $\hat{\boldsymbol{\Theta}}_{PS, j}$ can be written as:
\begin{align} \label{est1}
 \hat{\boldsymbol{\Theta}}_{PS, j} = \boldsymbol{A}_{j} (\boldsymbol{\Theta}_{rep} + \boldsymbol{\Theta}_{\textit{diff},j}) + \boldsymbol{N}_{j}^{'},  
\end{align}
where $\boldsymbol{\Theta}_{\textit{diff},j}$ is defined as the difference between $\boldsymbol{\Theta}_{rep} - \boldsymbol{\Theta}_{j} $.
Combining (\ref{est1}) for $j \in \{1, \dots, T\}$ and defining a total noise term $\boldsymbol{N}^{*}_{j} \triangleq \boldsymbol{A}_{j} \boldsymbol{\Theta}_{\textit{diff},j} + \boldsymbol{N}^{'}_{j}$, which represents the noise due to the channel, interference from the blind transmitters, and the difference between representative updates and the real updates, we obtain
\begin{equation} \label{rep_grad}
\hat{\boldsymbol{\Theta}}_{PS} = \boldsymbol{A} \boldsymbol{\Theta}_{rep} + \boldsymbol{N}^{*},
\end{equation}
where $\hat{\boldsymbol{\Theta}}_{PS} = [\Delta \hat{\boldsymbol{\theta}}_{PS}(t-T+1); \cdots ;\Delta \hat{\boldsymbol{\theta}}_{PS}(t)] \in \mathbb{R}^{T \times 2N}$ and `;' represents row-wise concatenation. Additionally, we have $\boldsymbol{A} \in \{0,1\}^{T \times M}$, $\boldsymbol{\Theta}_{rep} \in \mathbb{R}^{M \times 2N} $ and $\boldsymbol{N}^{*} \in \mathbb{R}^{T \times 2N}$.

By solving LSE for (\ref{rep_grad}), we can get an estimate for the representative updates as $\hat{\boldsymbol{\Theta}}_{rep}$. Using this representation, the PS can infer the characteristics of the users' data distribution as a similarity between user representations, which can then be utilized in the user selection procedure.

Due to the limited and stochastic energy arrivals, some users might dominate the training and create a bias towards certain labels and certain users.
By employing \textit{cosine similarity}, users are clustered into groups to promote diverse user contributions by selecting the expected number of users based on energy distribution from each cluster to ensure unbiased training. This approach helps to reduce the bias due to the non-i.i.d data and provides fair performance among clients, as noted similarly in \cite{wang2020rl, balakrishnan2022diverse}.

\section{Numerical Results}

In this section, we evaluate the performance of our proposed user scheduling methods across multiple scenarios. We consider image classification tasks on the MNIST \cite{deng2012mnist}, FMNIST \cite{xiao2017fashionmnistnovelimagedataset}, and CIFAR-10 \cite{cifar10} datasets under non-i.i.d. data distributions. 
For MNIST and FMNIST, we use a single-layer neural network with 784 input and 10 output neurons ($2N = 7850$). For CIFAR-10, we use a convolutional neural network (CNN) ($2N = 797962$) as in \cite{acar2021federated}. Training is performed by SGD with a learning rate of 0.05 and a scheduler, $\tau = 5$ and mini-batch size  $|\xi_m(t)| = 100$ for MNIST and FMNIST, and $\tau = 3$ and $|\xi_m(t)| = 128$ for CIFAR-10.

\begin{figure}
    \centering
    \begin{subfigure}{0.49\columnwidth}
        \centering
        \includegraphics[trim={0.4cm 0.0cm 0.37cm 0.3cm},clip,width=\textwidth]{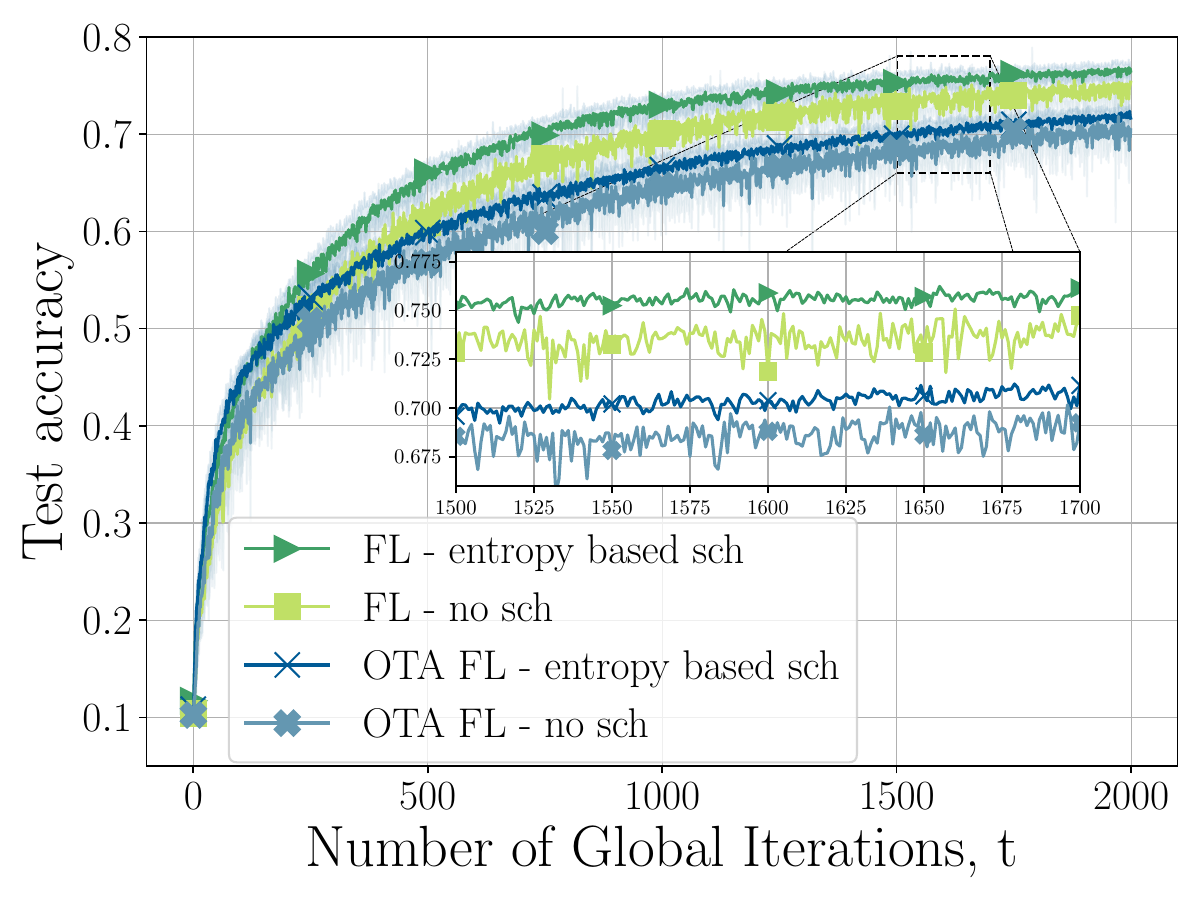}
        \vspace{-0.6cm}
        \captionsetup{font=scriptsize}
        \caption{$\beta = 0.1$.}
        \label{fig:ci100_dir01}
    \end{subfigure}
    \begin{subfigure}{0.49\columnwidth}
        \centering
        \includegraphics[trim={0.4cm 0.0cm 0.37cm 0.3cm},clip,width=\textwidth]{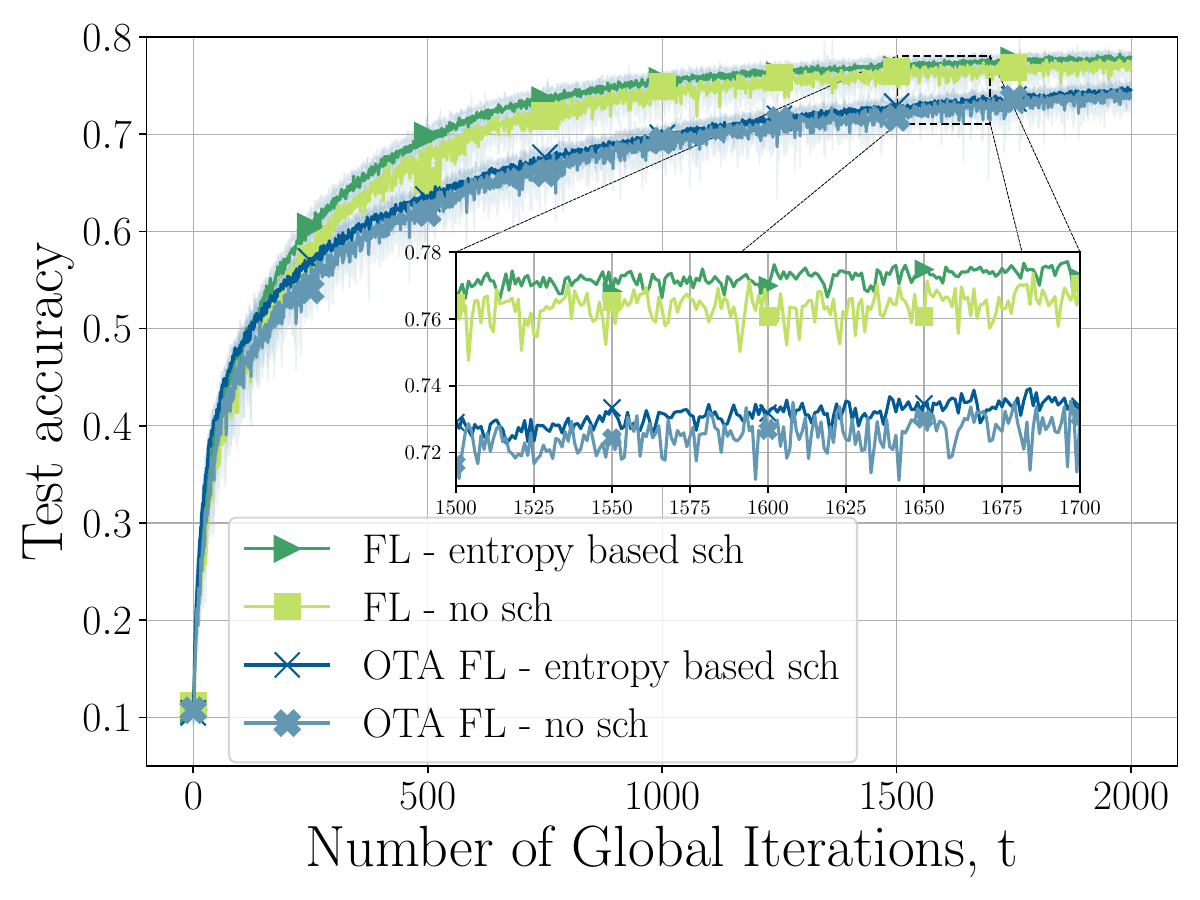}
        \vspace{-0.6cm}
        \captionsetup{font=scriptsize}
        \caption{$\beta = 0.2$.}
        \label{fig:ci100_dir02}
    \end{subfigure}
    \vspace{-0.5cm}
    \captionsetup{font=footnotesize}
    \caption{The mean test accuracy of entropy-based scheduling for CIFAR-10 with $M = 100$,
    $\left| \mathcal{B}_{m} \right| = 500$ and $p_e^{m}(t) = 0.1$, $\forall m, t$.}
    \label{fig:fig_ci100_dir0102}
    \vspace{-0.7cm}
\end{figure}

To simulate highly non-i.i.d. data, we consider two different distribution scenarios. In the first scenario, users receive data limited to a fixed number of labels, either 1 or 2 classes assigned per user. In the second one, we sample $ \boldsymbol{p_m} \sim \text{Dir}_{N_c}(\beta)$ with $ \boldsymbol{p_m} = [p_{m,0} , \cdots, p_{m,N_c-1}]$, and user  $m$  receives  $p_{m,{n_c}}$ portion of its data from class  $n_c \in [N_c]$. $\beta$ is a Dirichlet distribution parameter, where smaller values of $\beta$ lead to more unbalanced partitions.

We evaluate the performance of our wireless OTA FL with highly non-i.i.d data against a baseline scenario without user scheduling, where users participate in the global learning process whenever they have energy.
Throughout the simulations, users are connected to a PS through wireless fading MAC where channel gains from each user to each PS antenna are i.i.d., with parameters $K = 200$, $\sigma_{h}^{2} = 1$, and $\sigma_{z}^{2} = 0.1$.

In Fig. \ref{fig:fig_ci100_dir0102}, we demonstrate the performance of entropy-based scheduling for CIFAR-10 dataset with $\beta \in \{0.1,0.2\}$  for $M = 100$ users with $\left| \mathcal{B}_{m} \right| = 500$ and $p_e^{m}(t) = 0.1$ for $m \in [M]$. We observe that the gains from our scheme are more significant in scenarios with greater heterogeneity for both the error-free and OTA FL cases. As shown in the figure, as the distribution becomes more heterogeneous as in Fig \ref{fig:ci100_dir01}, (i.e., $\beta=0.1$), the impact of entropy-based scheduling increases. 

Fig. \ref{fig:fig_mn40_niid12} shows the mean test accuracies for data distributions with 1 class and 2 classes per user on the MNIST dataset with $M = 40$, $\left| \mathcal{B}_{m} \right| = 1250$, $p_e^{m}(t) = 0.25$ for $m \in [M]$ and estimation phases of $T=100$ and $T=200$ iterations. In both cases, entropy-based scheduling results in higher and more stable accuracy levels. For the unknown data distribution, the PS estimates local user representations after $T$ iterations and groups users into 10 clusters,
scheduling one user per cluster for each iteration, similar to \cite{wang2020rl}. Similar to the previous case, our scheme achieves higher gains in scenarios with greater heterogeneity. One reason for this is that estimation becomes more difficult when users are more similar, i.e., with less heterogeneous distribution, clustering based on the estimations becomes harder as well. However, scheduling still offers a performance gain compared to no scheduling baseline, even when data distributions are less heterogeneous.

In Fig. \ref{fig:fig_mn20_fas40}, we illustrate the performance of our scheduling policies with 10 clusters for the MNIST dataset with $M = 20$ users, where $\left| \mathcal{B}_{m} \right| = 2500$, $p_e^{m}(t) = 0.5$ and FMNIST dataset with $M = 40$ users, where $\left| \mathcal{B}_{m} \right| = 1250$, $p_e^{m}(t) = 0.25$ for $m \in [M]$. For known data distributions, the entropy-based scheduling method outperforms the no-scheduling baseline approach.  
For unknown data distributions, the global model's performance improves after the estimation phase, approaching the entropy-based case.

The results above demonstrate that our entropy-based scheduling approach, leveraging users' data distributions, gives a performance boost for OTA FL with EH devices over a wireless MAC channel with blind transmitters, yielding higher and more stable accuracy levels. Moreover, the results demonstrated that in cases of unknown data distributions, user representations can be estimated on the PS side to schedule diverse users, preserve privacy, eliminate redundant update transfers, and improve learning performance. In both entropy-based and LSE-based methods, we select diverse users to achieve a uniform data distribution, closely approximating full participation updates and minimizing error in (\ref{epsilon}).

\begin{figure}
    \centering
    \vspace{0.026cm}
    \begin{subfigure}{0.49\columnwidth}
        \centering
        \includegraphics[trim={0.4cm 0.2cm 0.37cm 0.3cm},clip,width=\textwidth]{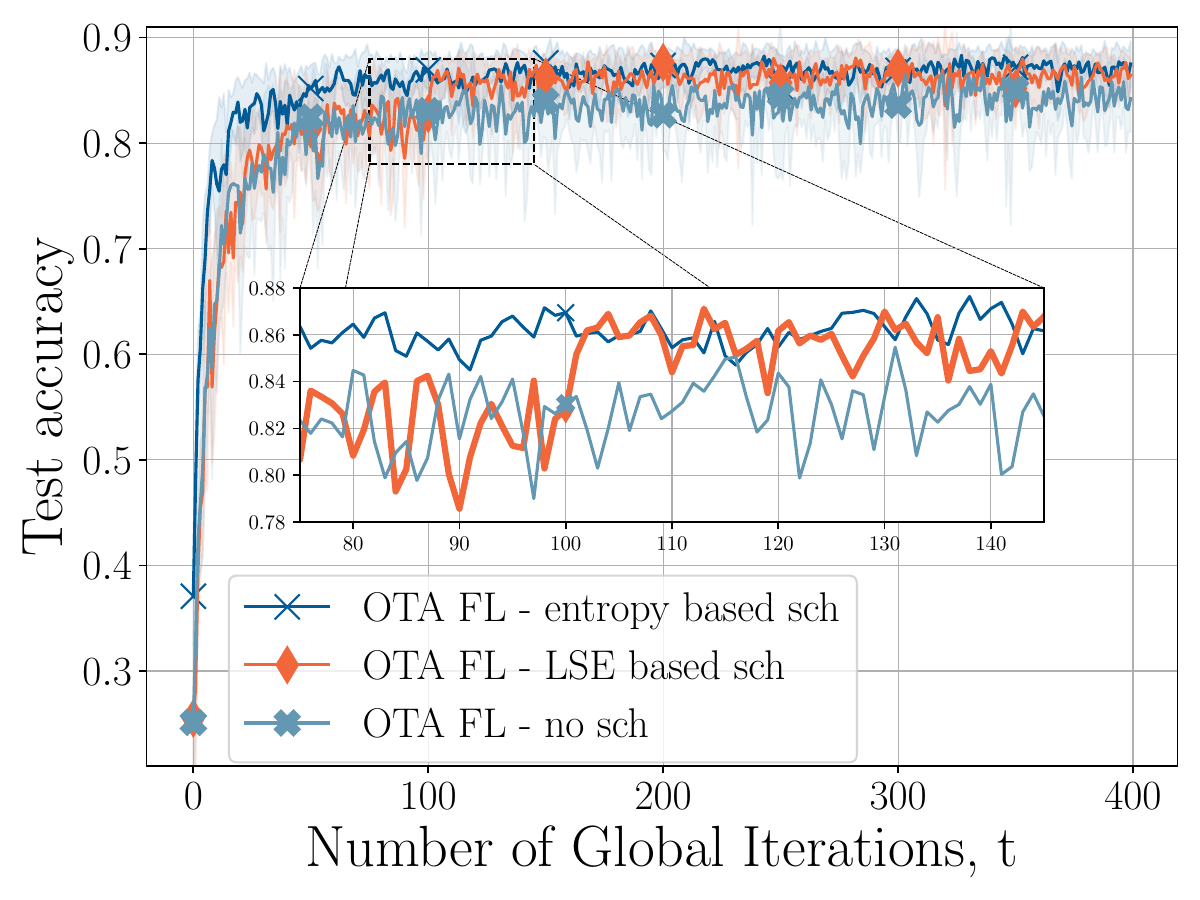}
        \vspace{-0.6cm}
        \captionsetup{font=scriptsize}
        \caption{1 class per user and $T=100$.}
        \label{fig:mn40_niid1}
    \end{subfigure}
    \begin{subfigure}{0.49\columnwidth}
        \centering
        \includegraphics[trim={0.4cm 0.2cm 0.37cm 0.3cm},clip,width=\textwidth]{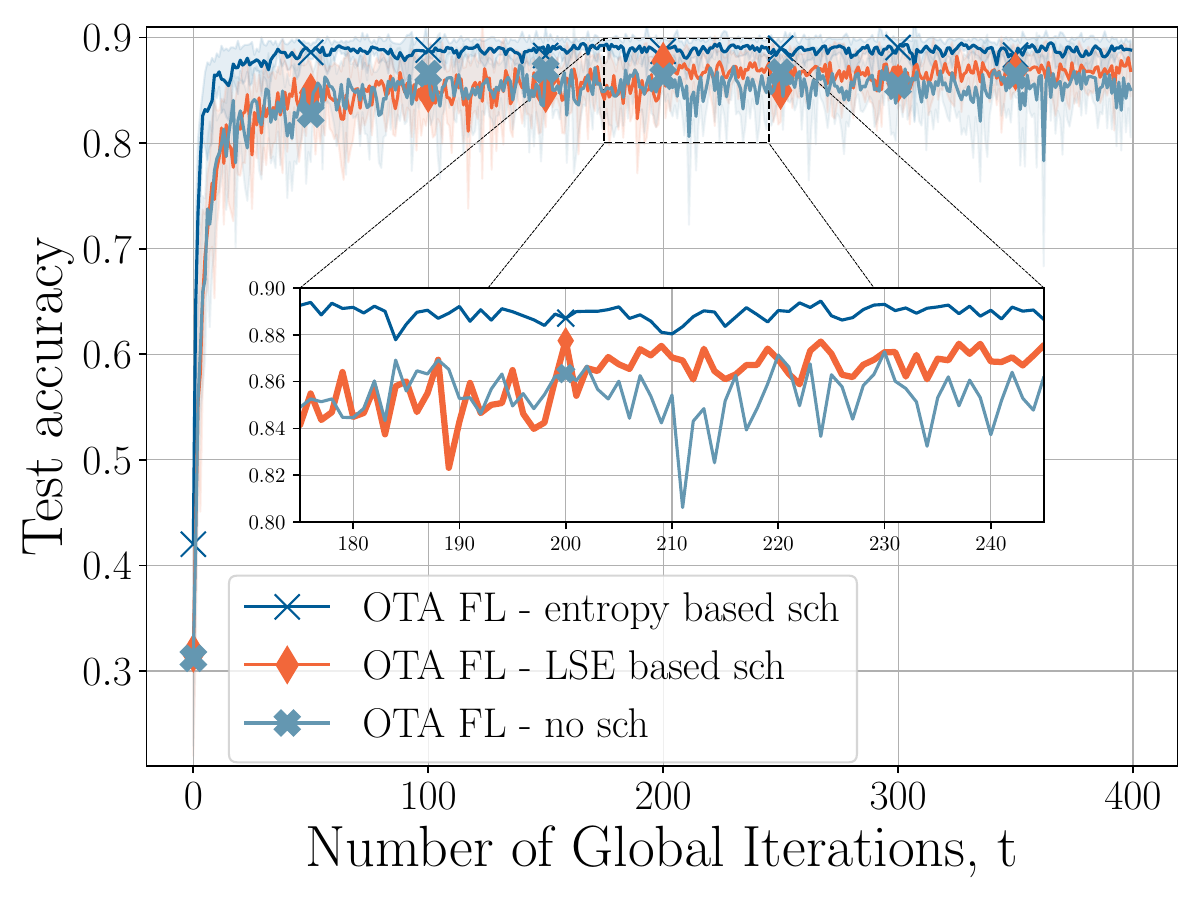}
        \vspace{-0.6cm}
        \captionsetup{font=scriptsize}
        \caption{2 class per user and $T=200$.}
        \label{fig:mn40_niid2}
    \end{subfigure}
    \vspace{-0.5cm}
    \captionsetup{font=footnotesize}
    \caption{The mean test accuracy for MNIST with $M = 40$, $\left| \mathcal{B}_{m} \right| = 1250$ and $p_e^{m}(t) = 0.25$, $\forall m, t$.}
    \label{fig:fig_mn40_niid12}
    \vspace{-0.35cm}
\end{figure}
\begin{figure}
    \centering
    \begin{subfigure}{0.49\columnwidth}
        \centering
        \includegraphics[trim={0.4cm 0.2cm 0.37cm 0.3cm},clip,width=\textwidth]{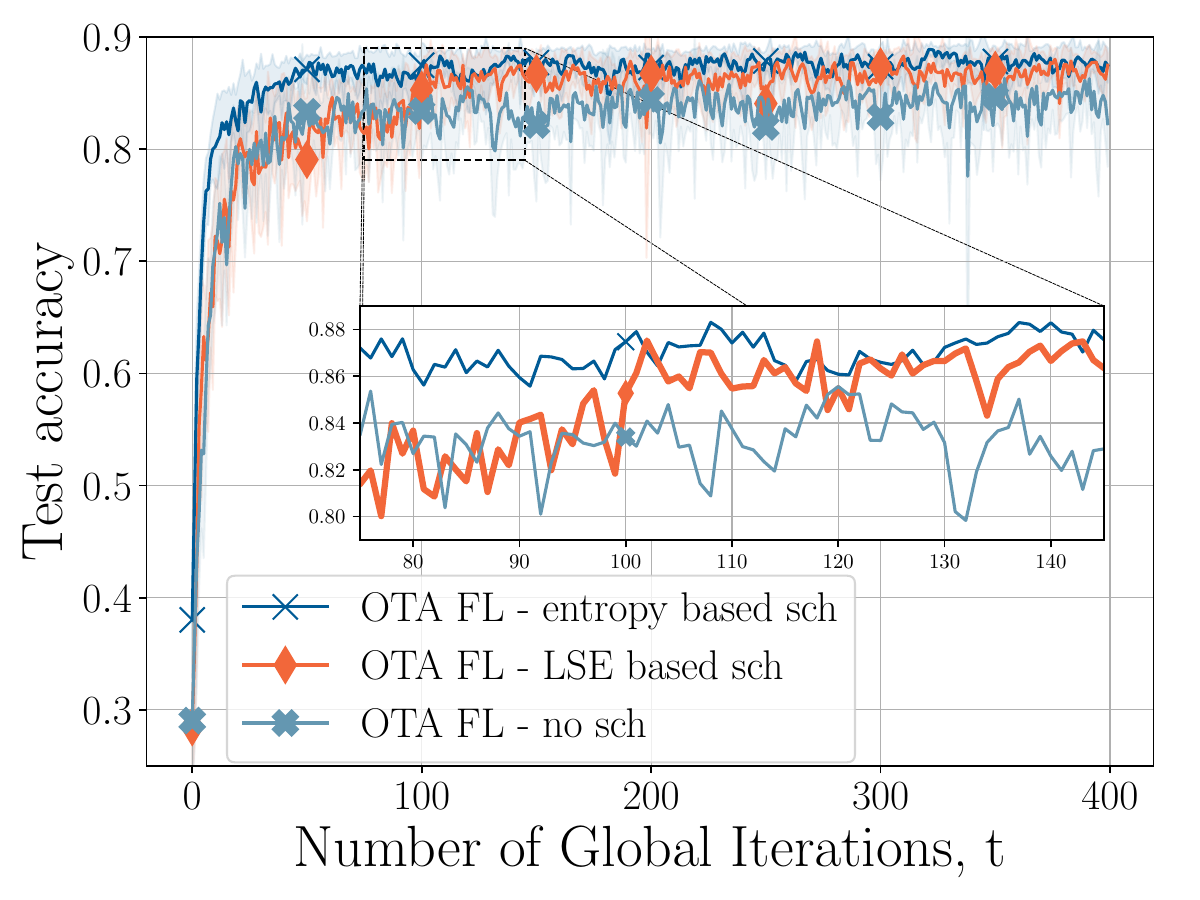}
        \vspace{-0.6cm}
        \captionsetup{font=scriptsize}
        \caption{MNIST, $M = 20$, 1 class per user and $T=100$.}
        \label{fig:mn20_niid1}
    \end{subfigure}
    \begin{subfigure}{0.49\columnwidth}
        \centering
        \includegraphics[trim={0.4cm 0.2cm 0.37cm 0.3cm},clip,width=\textwidth]{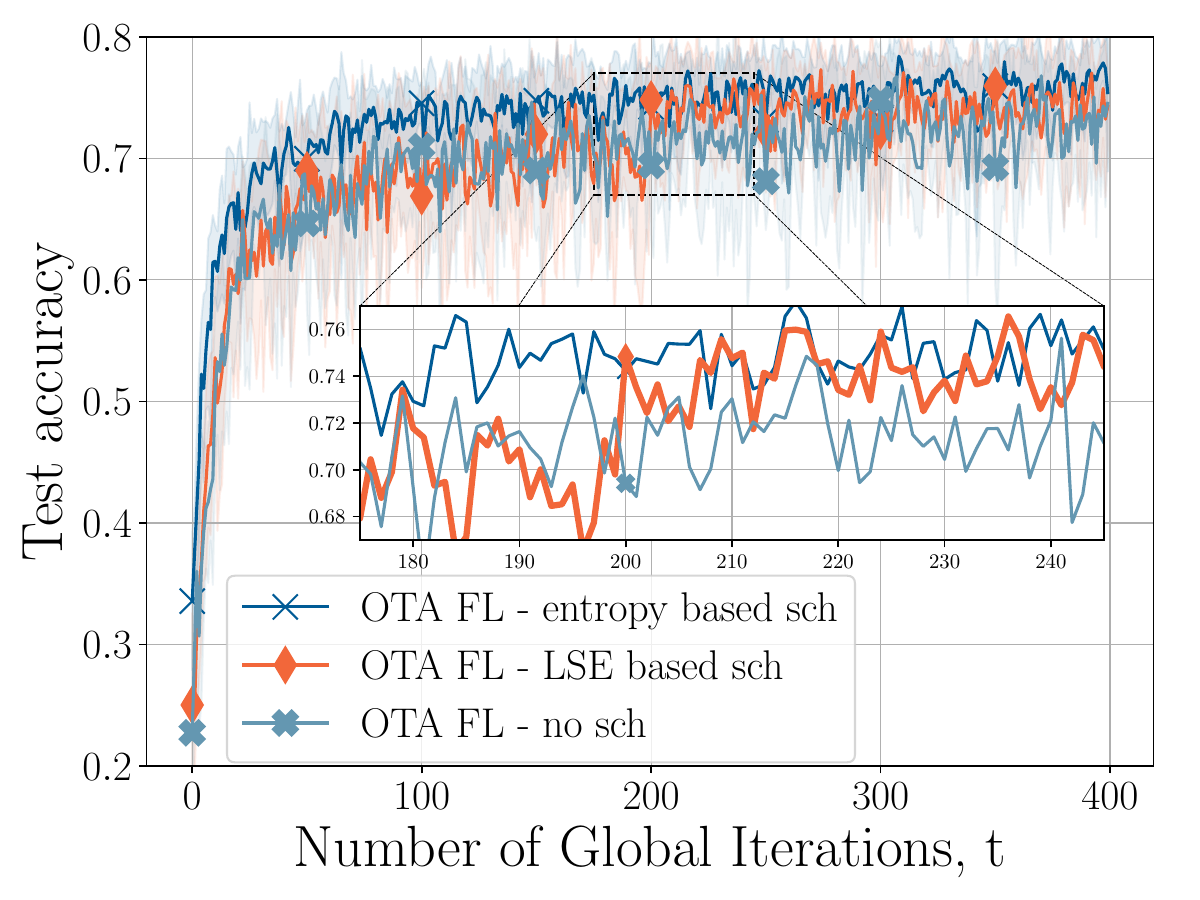}
        \vspace{-0.6cm}
        \captionsetup{font=scriptsize}
        \caption{FMNIST, $M = 40$, 1 class per user and $T=200$.}
        \label{fig:fas40_niid1}
    \end{subfigure}
    \vspace{-0.5cm}
    \captionsetup{font=footnotesize}
    \caption{The mean test accuracy for MNIST and FMNIST.}
    \label{fig:fig_mn20_fas40}
    \vspace{-0.7cm}
\end{figure}

\section{Conclusions}

In this work, we study user scheduling and user representation estimation for OTA FL with EH devices over wireless fading MAC with highly heterogeneous data distribution. For both known and unknown data distribution scenarios, our framework proposes user scheduling based on the relations of users. We analyze the convergence rate for the OTA FL with EH devices and demonstrate the effect of user scheduling. While our first proposed entropy-based approach uses label data distribution information to reduce the error in the convergence bound, our second least squares estimation approach enables the PS to infer user relationships for scheduling without compromising privacy. The numerical results from various setups demonstrate that performing user scheduling and update estimation significantly enhances performance, especially in highly heterogeneous distributions. A potential future direction is to implement clustered federated learning for the user clusters derived from our estimation.

\appendices
\section*{Appendix A} 
We define:
\begin{align} 
    \boldsymbol{w}(t+1) &\triangleq \boldsymbol{\theta}_{PS}(t) + \frac{1}{\left| \mathcal{S}(t) \right| }\sum_{m\in \mathcal{S}(t)}\Delta \boldsymbol{\theta}_{m}(t), \label{aux_w} \\
    \boldsymbol{v}(t+1) &\triangleq \boldsymbol{\theta}_{PS}(t) + \frac{1}{M} \sum_{m=1}^{M} \Delta \boldsymbol{\theta}_{m}(t). \label{aux_v}
\end{align}
From (\ref{gl_update_part}), we have $    \boldsymbol{\theta}_{PS}(t+1)=\boldsymbol{\theta}_{PS}(t)+ \Delta \hat{\boldsymbol{\theta}}_{PS}(t)$. Using this, we can derive 
\begin{align} \label{app1}
&\hspace {-2pc}\left \| \boldsymbol{\theta}_{PS}(t+1) - \boldsymbol{\theta}^* \right \|_{2}^{2} \notag\\ 
=& \left \| \boldsymbol{\theta}_{PS}(t+1) - \boldsymbol{w}(t+1) + \boldsymbol{w}(t+1) - \boldsymbol{\theta}^* \right \|_{2}^{2} \notag\\ 
=& \left \| \boldsymbol{\theta}_{PS}(t+1) - \boldsymbol{w}(t+1) \right \|_{2}^{2} + \left \| \boldsymbol{w}(t+1) - \boldsymbol{\theta}^* \right \|_{2}^{2} \notag\\ 
&+\, 2\langle \boldsymbol{\theta}_{PS}(t+1) - \boldsymbol{w}(t+1), \boldsymbol{w}(t+1) - \boldsymbol{\theta}^* \rangle.
\end{align}
To bound these terms, we employ the following lemmas.

\begin{lemma} \label{lemma1}
For the first and third terms in (\ref{app1}), we have
{\small{
    \begin{align*}&\mathbb {E} \left [{ \left \|{ \boldsymbol {\theta }_{PS}(t+1) - {\boldsymbol{w}} (t+1) }\right \|_{2}^{2} }\right] \le \frac { \eta ^{2}(t) \tau ^{2} G^{2}}{K} + \frac {\sigma _{z}^{2}N}{\alpha _{t}^{2} K {\left| \mathcal{S}(t) \right|} \sigma _{h}^{2}} \label{bound_term1} ,
    \end{align*}}}
    and
    \begin{equation*} \mathbb {E} \big [\langle \boldsymbol{\theta}_{PS}(t+1) - \boldsymbol{w}(t+1), \boldsymbol{w}(t+1) - \boldsymbol{\theta}^* \rangle \big] = 0.\end{equation*}
\end{lemma}
\begin{proof}
    The proofs are similar to Lemmas 1 and 3 in \cite{amiri2020BFL}.
\end{proof}
For the second term in (\ref{app1}), we proceed as follows:
\begin{align} \label{app_rec}
&\hspace {-2pc}\left \| \boldsymbol{w}(t+1) - \boldsymbol{\theta}^* \right \|_{2}^{2} \notag\\ 
=& \left \| \boldsymbol{w}(t+1) - \boldsymbol{v}(t+1) + \boldsymbol{v}(t+1) - \boldsymbol{\theta}^* \right \|_{2}^{2} \notag\\ 
=& \left \| \boldsymbol{w}(t+1) - \boldsymbol{v}(t+1) \right \|_{2}^{2} + \left \| \boldsymbol{v}(t+1) - \boldsymbol{\theta}^* \right \|_{2}^{2} \notag\\ 
&+\, 2\langle \boldsymbol{w}(t+1) - \boldsymbol{v}(t+1), \boldsymbol{v}(t+1) - \boldsymbol{\theta}^* \rangle.
\end{align}
\begin{lemma} \label{lemma2}
    For the second term in (\ref{app_rec}), we have
    \begin{align}
    & \mathbb {E} \left [{ \left \|{ \boldsymbol{v} (t+1) - {\boldsymbol {\theta }}^{*} }\right \|_{2}^{2} }\right] \notag\\
 & \qquad  \le\left ({1 - \mu \eta (t)~\left ({\tau - \eta (t) (\tau - 1) }\right) }\right) \mathbb {E} \left [{ \left \|{ \boldsymbol {\theta }_{PS}(t)\,\,- {\boldsymbol {\theta }}^{*} }\right \|_{2}^{2} }\right] \notag\\
    & \qquad \qquad + \left ({1+ \mu (1- \eta (t)) }\right) \eta ^{2}(t) G^{2} \frac {\tau (\tau -1)(2\tau -1)}{6} \notag\\
    & \qquad \qquad+\, \eta ^{2}(t) (\tau ^{2} + \tau -1) G^{2} + 2 \eta (t) (\tau - 1) \Gamma  \label{app2}.
    \end{align}
\end{lemma}
\begin{proof} 
The proof follows the same argument in  \cite[Lemma 2]{amiri2020BFL}. 
\end{proof}

\begin{lemma} \label{lemma3}
For the first term in (\ref{app_rec}), we have
\begin{align*}
    &\mathbb{E} \left[ \left\| \boldsymbol{w}(t+1) - \boldsymbol{v}(t+1) \right\|_{2}^{2} \right] \\
    &  \qquad \qquad = \left( \eta^2(t) \tau(\tau-1)LG + \eta(t)\tau \epsilon  \right)^{2}.
\end{align*}
\end{lemma}
\begin{proof}
    The proof is similar to \cite[Lemma 1]{balakrishnan2022diverse}.
\end{proof}

\begin{lemma} \label{lemma4}
The third term in (\ref{app_rec}) is bounded by
\begin{align}
    &\mathbb{E} \left[     2\langle \boldsymbol{w}(t+1) - \boldsymbol{v}(t+1), \boldsymbol{v}(t+1) - \boldsymbol{\theta}^* \rangle \right] \notag\\
    &\qquad \qquad \leq \left( \eta^2(t) \tau(\tau-1)LG + \eta(t)\tau \epsilon \right) c,
\end{align}
for some constant $c \geq 0$, which is related to $\Gamma$, $G$  and \(\mu\).
\end{lemma}
\begin{proof}
The proof follows a similar line as Lemma 1 in \cite{balakrishnan2022diverse}. 
\end{proof}

\bibliographystyle{IEEEtran}
\bibliography{ref.bib}

\end{document}